\RequirePackage{amsthm}
\documentclass[a4paper,12pt]{article}
\usepackage[margin=1in]{geometry}
\usepackage{graphicx} 
\usepackage{bbm}
\usepackage{amsfonts,amsmath,amssymb}
\usepackage{url, hyperref}
\usepackage{xcolor}
\usepackage{pifont}
\usepackage{enumitem}
\usepackage{multirow}%
\usepackage{mathrsfs}%
\usepackage{manyfoot}%
\usepackage{booktabs}%
\usepackage{algorithm}%
\usepackage{algorithmicx}%
\usepackage{algpseudocode}%

\DeclareMathOperator*{\argmin}{arg\,min\ }

\newtheorem{theorem}{Theorem}
\newtheorem{proposition}{Proposition}%
\newtheorem{lemma}{Lemma}%

\newtheorem{corollary}{Corollary}
\newtheorem{remark}{Remark}%
\newtheorem{definition}{Definition}%

\newcommand{\one}{\mathbbm{1}}
\DeclareMathOperator{\Tr}{Tr}
\newcommand{\R}{\mathbb{R}}

\newcommand{\cmark}{\ding{51}}%
\newcommand{\wmark}{\ding{55}}

\newcommand{\x}{\mathbf{x}}
\newcommand{\X}{\mathbf{X}}
\newcommand{\w}{\mathbf{w}}
\newcommand{\y}{\mathbf{y}}
\newcommand{\cb}{\mathbf{c}}
\newcommand{\z}{\mathbf{z}}

\newcommand{\yt}{\tilde{\mathbf{y}}}

\newcommand{\wpr}{\mathbf{w}_{proj}}

\DeclareMathOperator{\CH}{CH}
\DeclareMathOperator{\DT}{DT}

\renewcommand{\emph}[1]{\textit{#1}}

\newcommand{\unemph}[1]{\textup{#1}}

\begin{document}

\title{Locality Regularized Reconstruction: Structured Sparsity and Delaunay Triangulations}


\author{\hspace{100pt}Marshall Mueller\footnote{Department of Mathematics, Tufts University, Medford, MA 02155, USA (\url{marshallm@protonmail.ch}, \url{jm.murphy@tufts.edu}, \url{abiy.tasissa@tufts.edu})}  \footnote{Corresponding author}
 \hspace{95pt}\vspace{5pt}
 \and \hspace{-0.5em}James M. Murphy\footnotemark[1] \and \hspace{-20pt} \hspace{-0.3em}Abiy Tasissa\footnotemark[1] }
\maketitle

\begin{abstract}
    Linear representation learning is widely studied due to its conceptual simplicity and empirical utility in tasks such as compression, classification, and feature extraction.  Given a set of points $[\x_1, \x_2, \ldots, \x_n] = \X \in \R^{d \times n}$ and a vector $\y \in \R^d$, the goal is to find coefficients $\w \in \R^n$ so that $\X \w \approx \y$, subject to some desired structure on $\w$. 
In this work we seek $\w$ that forms a local reconstruction of $\y$ by solving a regularized least squares regression problem.  We obtain local solutions through a locality function that promotes the use of columns of $\X$ that are close to $\y$ when used as a regularization term.  We prove that, for all levels of regularization and under a mild condition that the columns of $\X$ have a unique Delaunay triangulation, the optimal coefficients' number of non-zero entries is upper bounded by $d+1$, thereby providing local sparse solutions when $d \ll n$.  Under the same condition we also show that for any $\y$ contained in the convex hull of $\X$ there exists a regime of regularization parameter such that the optimal coefficients are supported on the vertices of the Delaunay simplex containing $\y$.  This provides an interpretation of the sparsity as having structure obtained implicitly from the Delaunay triangulation of $\X$.  We demonstrate that our locality regularized problem can be solved in comparable time to other methods that identify the containing Delaunay simplex.
\end{abstract}

\section{Introduction}

In various applications such as genomics \cite{libbrecht2015machine} and natural language processing \cite{vaswani2017attention}, the underlying data is high-dimensional and characterized by significant noise. Data in this raw form is not conducive to numerical computation or analysis. In light of this, feature extraction methods play a crucial role in acquiring representations of the data that facilitate its use in subsequent tasks while preserving certain aspects of the original data. These representations are commonly realized as points in Euclidean space and often exhibit desirable features such as low-dimensionality and sparsity. Widely employed methods for this purpose include principal component analysis and its generalizations \cite{hotelling1933analysis, hyvarinen2000independent, candes2011robust}, manifold learning techniques \cite{scholkopf1997kernel, tenenbaum2000global, roweis2000nonlinear, belkin2003laplacian, coifman2006diffusion}, as well as tailored features for specific applications (e.g., histogram of gradients (HOG) and scale invariant feature transform (SIFT) in image processing using local information for each pixel  \cite{dalal2005histograms,lowe1999object}).  In the regime where large volumes of labeled training data are available, neural network architectures such as convolutional neural networks (CNN) \cite{krizhevsky2012imagenet,hershey2017cnn} and transformers \cite{vaswani2017attention, wolf2020transformers} have been utilized for representation learning.  In this paper, we focus on sparse features in the setting of structured sparse coding.

Sparse coding is a representation learning model that posits that data can be expressed as a combination of vectors known as dictionary atoms. The coefficients of this combination are assumed to be sparse. Specifically, given a predefined matrix $\mathbf{X}=[\mathbf{x}_1, \mathbf{x}_2, \ldots, \mathbf{x}_n] \in \mathbb{R}^{d \times n}$, the objective for linear sparse coding is to represent a vector $\mathbf{y} \in \mathbb{R}^d$ as a linear combination of $n$ vectors, approximating $\mathbf{y}$ as $\mathbf{X}\mathbf{w}$, where $\mathbf{w}\in\mathbb{R}^{n}$ is a sparse vector. A vector $\mathbf{w}$ is considered $k$-sparse if $\Vert\mathbf{w}\Vert_0 \le k\le n$, where $\Vert\cdot\Vert_0$ denotes the $\ell_0$ quasi-norm.  This counts the cardinality of the support of $\w$,  namely the number of non-zero entries. Various sparsity-promoting regularizers can be employed for recovering a sparse representation including $\ell_p$ norm regularizers such as $\ell_0$ \cite{cai2011orthogonal, needell2009cosamp}, $\ell_1$ \cite{candes2006near, candes2006stable}, $\ell_p$ for $0<p\leq 1$ \cite{chartrand2007exact,foucart2009sparsest}, as well as entropy minimization \cite{huang2018sparse}. In structured sparse coding, additional prior knowledge about the sparse codes is available. For instance, partial support of $\w$ may be known a priori, or there could be group or hierarchical structures in the sparse coefficients
\cite{khajehnejad2009weighted,vaswani2010modified,huang2010benefit,yuan2006model,sprechmann2010collaborative,foucart2011recovering}. In \cite{roweis2000nonlinear,elhamifar2011sparse}, another form of structured sparsity---local sparsity---is considered. Therein, the idea is to promote representation of a data point using nearby dictionary atoms. The present paper specifically focuses on this notion of sparsity under the assumption that data points are generated from
structured triangulations. 

Our model for structured sparse coding is grounded in Delaunay triangulations \cite{delaunay1934sphere}. These triangulations have found widespread application in computational geometry \cite{cheng2016delaunay}, computational fluid dynamics \cite{weatherill1992delaunay}, and geometric scale detection \cite{gillette2022data}.
Moreover, when using linear interpolation, the optimality of these triangulations with respect to minimal approximation error has been studied \cite{omohundro1989delaunay,chen2004optimal}. In our context, the dictionary atoms define a unique Delaunay triangulation, with each data point lying within the convex hull of these atoms. In this framework, exact local sparsity entails representing a data point using only the vertices of the simplex to which it belongs. The coefficients, by construction, lie within the probability simplex $\Delta^{n}:=\{\w=(w_1,w_2,\dots,w_n) \ | \ w_{i}\ge 0 \ \forall i, \  \sum_{i=1}^{n}w_{i}=1\}$.  For a fixed dictionary $\mathbf{X}$, the \emph{locality} of a representation parameterized by $\mathbf{y}$ is defined as
\begin{equation*}
    \ell_\y(\w) := \sum_{i=1}^n w_i \Vert \x_i - \y \Vert^2. 
\end{equation*}
Locality encourages non-zero values for $w_i$ only when $\mathbf{x}_i$ is in close proximity to $\mathbf{y}$. When there exists at least one $\mathbf{w}$ such that $\mathbf{X} \mathbf{w} = \mathbf{y}$, the \emph{exact} locality regularized coding problem can be considered:
\begin{align} \label{eqn:exact}
        \argmin_{\w \in \Delta^n} \quad  \sum_{i=1}^n w_i \Vert \x_i - \y \Vert^2 \quad \text{s.t.} \quad \X\w = \y. \tag{E} 
\end{align}
Under mild assumptions, \cite{tasissa2023k} demonstrates that the optimal solution to \eqref{eqn:exact} is supported on the vertices of the Delaunay simplex containing $\mathbf{y}$; we refer to Theorem \ref{thm:TSP_2023_E} for a precise statement. In many applications, the underlying data is susceptible to noise, rendering the exact problem inapplicable within that context. Consequently, one can relax the exact problem to the locality regularized least squares problem for some balance parameter $\rho>0$ as follows:
\begin{align}\label{eqn:relaxed}
\argmin_{\w \in \Delta^n} \quad \frac{1}{2}\lVert \mathbf{X}\mathbf{w} - \mathbf{y}\rVert^2 + \rho \sum_{i=1}^n w_i \lVert \mathbf{x}_i - \mathbf{y} \rVert^2.\tag{R}
\end{align}
The central focus of this paper is to assert that the aforementioned relaxation (\ref{eqn:relaxed}) preserves the advantages of the exact problem (\ref{eqn:exact}). Indeed, we will establish that provable guarantees for sparse recovery remain unaltered with an appropriate choice of $\rho$. In \cite{tasissa2023k}, the main algorithm utilizes a regularized objective, similar to other dictionary learning algorithms that incorporate objectives such as $\ell_1$ regularization. The attractiveness of employing this regularization lies in its ability to facilitate the use of first-order methods \cite{beck2017first}, offering computational advantages, particularly for large-scale problems. Hence, our focus is directed towards establishing the theoretical properties of this computationally preferred problem.

\subsection{Summary of Contributions}
We present and analyze the regularized minimization problem, \eqref{eqn:relaxed}, to find a local and sparse representation of $\y$ in terms of $\X$. Our analysis only requires that $\X$ be in general position, which is a weak requirement in practice.\\ 

\begin{enumerate}[label=(\roman*)]
\item  Our main result is stated in Theorem \ref{thm:main_R}: when $\y$ lies in the convex hull of $\X$, $\eqref{eqn:relaxed}$ can be solved to identify the $d$-simplex of the Delaunay triangulation of $\X$ containing $\y$ when $\rho$ is less than a bound depending only on $\y$ and $\X$. 
Additionally, we show that the solution $\w_\rho$ to (\ref{eqn:relaxed}) is such that $\X\w_\rho$ converges linearly in $\rho$ to $\y$.\\ 

\item  When $\y$ is not contained in the convex hull of $\X$, we show that $\X \w_\rho$ converges linearly in $\rho$ to the projection of $\y$ onto the convex hull of $\X$.  Relatedly, for small enough $\rho$ we show that the solution $\w_\rho$ will have at most $d+1$ nonzero entries coinciding with the vertices of the $d$-simplex whose face contains the projection of $\y$ onto the convex hull of $\X$. \\

\item 
We establish a connection between the optimal sparse solution derived from our optimization and the task of determining the simplex within a Delaunay triangulation to which a point belongs, which is important in the context of interpolation \cite{chang2020algorithm}. We offer a perspective based on structured sparse recovery for this problem. We also demonstrate that our algorithm performs comparably to baseline algorithms in terms of efficiency.\\

\end{enumerate}

\noindent We provide code to solve \eqref{eqn:relaxed} efficiently using CVXOPT in Python and replicate our figures here: \url{https://github.com/MarshMue/LocalityRegularization}.

\subsection{Paper Organization}
The remainder of the paper is organized as follows.  In Section \ref{sec:background} we provide necessary theoretical background to understand our results as well as related work as it pertains to sparse representation and identifying the containing Delaunay simplex. Section \ref{sec:theory} provides the details of our theoretical analysis. We validate our theoretical results through various experiments in Section \ref{sec:experiments} before concluding in Section \ref{sec:conclusion}. 
 Details on the practical optimization of \eqref{eqn:relaxed} are provided in Appendix \ref{sec:app-comp}. 

\subsection{Notation}
We use boldface capitalized letters (e.g., $\X$) to denote matrices and boldface lower case letters (e.g., $\y$) to denote vectors. $\one_{d}$ denotes the vector of ones with length $d$.  $\Tr(\mathbf{X})$ denotes the trace of a matrix. Let $\CH(\X) = \{\y \in \R^d \ | \ \exists \w \in \Delta^n, \ \X \w = \y \}$ denote the convex hull of $\X$. The terms ``point'' and ``vector'' will be used interchangeably depending on the context.  We use $S\subset\mathbb{R}^{d}$ to denote the $(d+1)$ vertices that constitute a $d$-simplex $\CH(S)$; precise discussion of simplices is in Section \ref{sec:background}. Furthermore, we use $\partial S$ to denote the boundary of $\CH(S)$.  The $\ell^{2}$-norm will be written as $\Vert \cdot \Vert$ while other $\ell^{p}$ norms will be written as $\Vert \cdot\Vert_{p}$.

\section{Background}\label{sec:background}

\subsection{Representation Learning via Statistical Signal Processing}

Many practical problems in statistical signal processing can be modeled through the least squares
problem $\argmin_{\w\in \R^n} \frac{1}{2}\Vert\y - \X\w\Vert^2$. In the typical case $n>d$, the problem is an underdetermined
least squares problem which has infinitely many solutions. A common approach to address this is to introduce the regularizer $\Vert\w\Vert^2$ into the objective, resulting in the following problem: $\argmin_{\w\in \R^n} \frac{1}{2}\Vert\y - \X\w\Vert^2+\rho \Vert\w\Vert^2$ where $\rho>0$ is a regularization parameter. This regularization, known as \emph{$\ell_2$} or \emph{Tikhonov regularization}, ensures that there
is a unique solution for $\rho>0$ that admits a closed form. 

In many signal processing
applications, it is often assumed that, after an appropriate linear transformation, the underlying solution is sparse, indicating very few nonzero entries in the signal. Given that, a natural optimization program is based on the $\ell_0$ quasi-norm which solves $\argmin_{\w\in \R^n} \frac{1}{2}\Vert\y-\X\w\Vert^2+\rho\Vert\w\Vert_0$ and thereby enforces sparsity of $\w$ directly by making $\Vert\w\Vert_{0}$ small. However, the $\ell_0$ regularization is not amenable to optimization as it is generally NP-hard. Under certain assumptions on $\X$,
iterative greedy methods such as orthogonal matching pursuit (OMP) \cite{davis1997adaptive,pati1993orthogonal,tropp2007signal} can be employed for the $\ell_0$ minimization problem.
A computationally viable alternative is $\ell_1$ regularization, resulting in the optimization program $\argmin_{\w\in \R^n} \frac{1}{2}\Vert\y-\X\w\Vert^2+\rho\Vert\w\Vert_1$. This is known as the basis pursuit problem \cite{tropp2004greed,chen2001atomic}. Extensive studies have explored the recovery of exact or approximate solutions for the $\ell_1$ regularized problem under certain conditions on $\X$ \cite{foucart2013mathematical}. In some cases, additional prior information about the support of the sparse vector is available. For this setting, the works in
\cite{lian2018weighted,mansour2017recovery,vaswani2010modified} consider weighted $\ell_1$ minimization of the following form: 
\[
\underset{\w\in \R^n}{\argmin}\quad \frac{1}{2}\Vert\y-\X\w\Vert^2+\rho\Vert\w\Vert_{1,\mathbf{\alpha}} ,
\]
where $\Vert\w\Vert_{1,\mathbf{\alpha}}:=\sum_{i=1}^{n} \alpha_i \vert w_i \vert $ and 
$\mathbf{\alpha} \in \R^{n}$ denotes the vector of weights. In a broader context, the prototypical
optimization problem for a structured sparse recovery problem can be written as:
\[
\underset{\w\in \R^n}{\argmin}\quad \frac{1}{2}\Vert\y-\X\w\Vert^2+ \rho R(\w),
\]
where $R(\w)$ is the regularization function that promotes the recovery of solutions with the desired prior. We note that, while this paper focuses on linear sparse coding,
which posits a linear relationship between $\y$ and $\w$, the framework is adaptable to the non-linear setting
leading to the non-linear sparse coding problem \cite{ho2013nonlinear,werenski2022measure, do2023approximation, mueller2023geometrically}.

\subsection[Triangulations on Rd]{Triangulations on $\mathbb{R}^{d}$}

\begin{figure}
    \centering
    \includegraphics[width=0.99\linewidth]{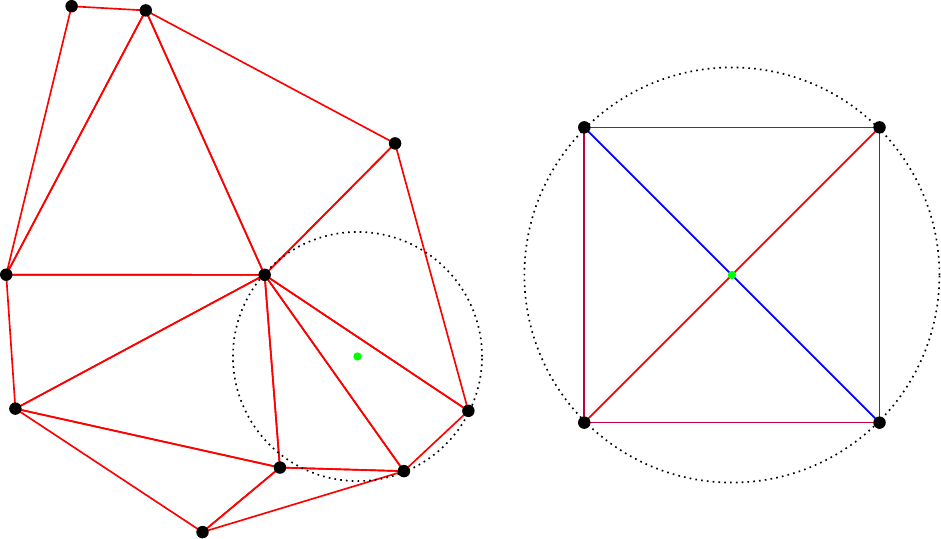}
    \caption{\emph{Left:} An example of a Delaunay triangulation in $\mathbb{R}^2$ with the empty circumscribing hypersphere condition satisfied for a queried triangle containing the green dot. 
\emph{Right:}  An example of a point configuration in $\mathbb{R}^{2}$ with non-unique Delaunay triangulation.  Either the red or blue edge together with the four outer edges generates a Delaunay triangulation. Note that the $4=d+2$ generating points are not in general position; they all lie on a common circle.}
    \label{fig:empty_circumscribing_hypersphere}
\end{figure}

Our approach to local sparsity is closely connected to triangulations in $\mathbb{R}^{d}$, which we review now. 

A \emph{$d$-simplex} is the convex hull of a set of $d+1$ points $\{\x_0,\x_1,..,\x_d\}\subset\R^d$. 
Concrete examples are a point (0-simplex) and a line segment (1-simplex). The $d+1$ points
that constitute the $d$-simplex are denoted as the \emph{vertices} of the simplex.  A \emph{$k$-face} of a $d$-simplex is the convex combination of a subset of $k+1$ vertices of the $d$-simplex.  Concrete examples include a point (a $0$-face), an edge (a $1$-face), a triangular face (a $2$-face), and the $d$-simplex itself (a $d$-face). A \emph{triangulation} of $\X$ is a set of $d$-simplices whose union is the convex hull of $\X$, and the intersection of any two $d$-simplices is either empty or a common $(d-1)$-face. 

\begin{definition}
A \unemph{Delaunay triangulation} of $\X$, $\DT(\X)$, is defined as a triangulation such that the circumscribing hypersphere of every $d$-simplex does not contain any other point of $\X$ in its interior. 
\end{definition}
A Delaunay triangulation $\DT(\X)$ always exists provided that $\X$ does not lie on any $k$-dimensional hyperplane for $k<d$. For $\DT(\X)$ to be unique, we must have that $\X$ lie in \emph{general position}: no $d+2$ points lie on the same hypersphere and the affine hull of $\X$ is $d$-dimensional. The requirement of $\X$ lying in general position is weak and is almost always satisfied by unstructured point clouds such as uniform samples taken from the hypercube. A notable exception to this is points representing the vertices of a regular grid, which will not be in general position as each grid cell will have all $2^d>d+1$ vertices lying on the same $d$-dimensional circumscribing hypersphere when $d\ge2$. This is important in the context of interpolation for scientific computing and may necessitate bespoke methods when working with resulting non-unique Delaunay triangulations \cite{chang2020algorithm}.  See Figure \ref{fig:empty_circumscribing_hypersphere} for examples. 

\subsubsection{Identifying A Delaunay Simplex}
Suppose that $\y$ lies in $\CH(\X)$ with a unique Delaunay triangulation $\DT(\X)$. One can then ask: what is the $d$-simplex of $\DT(\X)$ that contains $\y$? In this section we review two main existing ways to solve this problem.\\

\noindent\textbf{Convex Hull Linear Program:}\label{sec:CHLP}  $\DT(\X)$ can be identified by the \emph{lower faces} of the convex hull formed by ``lifting'' $\X$ via the lifting map $\x \mapsto (\x, \Vert\x \Vert^2) \in \R^{d+1}$ \cite{edelsbrunner1985voronoi}. 
A face of the convex hull is considered \emph{lower} if its inward normal has a positive component in the $(d+1)^{st}$ component. This connection can be utilized in order to determine the $d$-simplex containing a point $\y$; a point contained within a simplex will lie below the lower face corresponding to the $d$-simplex containing $\y$. This can be determined by solving a linear program that is interpretable as finding the first lower face of the convex hull that intersects with a ray shooting up in the $(d+1)^{st}$ dimension \cite{fukuda2004frequently,chang2020algorithm}. The linear program is stated as:
\begin{align}\label{eqn:convexHullLP}
    \min_{\cb \in \R^{n}, z \in \R} -\cb^T \y - z \quad \text{s.t.}\quad \begin{bmatrix}
        \X^T & \one_n
    \end{bmatrix} \begin{bmatrix}
        \cb \\ z
    \end{bmatrix}  \leq \mathbf{b},
\end{align}
where $b_i = \Vert\x_i\Vert^2$. We provide an illustration of this method in Figure \ref{fig:cvxLP}.

\begin{algorithm}
\caption{Convex Hull LP}\label{alg:convexHullLP}
\begin{algorithmic}[1]
\Require $\y \in \CH(\X)$
\State Solve \eqref{eqn:convexHullLP} to obtain $\cb^*, z^*$
\State Set simplex $S \leftarrow \{ \x_i : \x_i^T \cb^* + z^* = b_i \}$
\State \Return $S$
\end{algorithmic}
\end{algorithm}

\begin{figure}
    \centering
    \includegraphics[width=0.49\linewidth, trim=1cm 0.8cm 1cm 1.9cm]{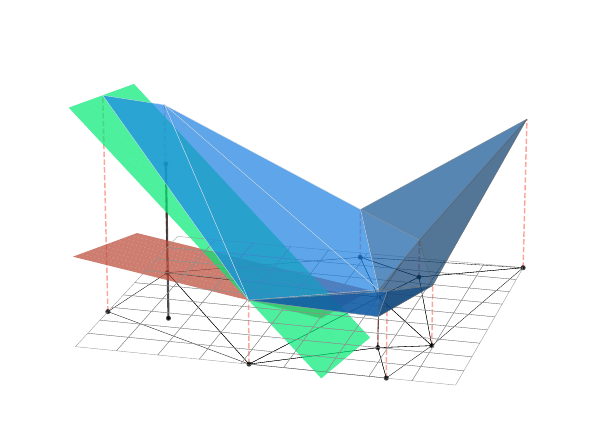}
    \includegraphics[width=0.49\linewidth]{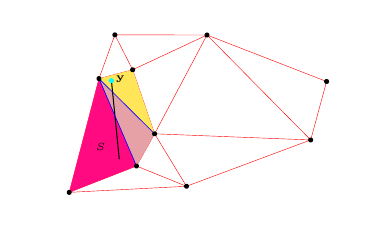}
    \caption{\textit{Left:} A visualization of the linear program described in Section \ref{sec:CHLP}. Both the 2D Delaunay triangulation and lower faces of the convex hull of the lifted points, coinciding with the triangulation, are shown. Points are connected to their lifted versions by dashed red lines. The red and green hyperplanes are interpretable as two solutions to \eqref{eqn:convexHullLP}; the green one represents an optimal solution while the red one represents a suboptimal solution. The objective measures the length of the ray, drawn as a black line, from $\y$ on the plane to the hyperplane. The constraints prohibit the ray from passing through a lower face and hence the hyperplane intersecting the lower face is optimal.  \textit{Right:} A viable visibility walk for DelaunaySparse. The colored triangles denote those on the path with $S$ labeling the initial triangle. A line is drawn through each triangle on the walk and the faces from which $y$ is visible are intersected and colored blue. }
    \label{fig:cvxLP}
\end{figure}

\noindent\textbf{DelaunaySparse}
DelaunaySparse \cite{chang2020algorithm} operates by first constructing a valid $d$-simplex, $S$, of $\DT(\X)$ and then flips from $d$-simplex to $d$-simplex until the one containing $\y$ is found. Each new $d$-simplex to flip to is determined by picking a face $F \subset S$ that $\y$ is \emph{visible} from, which means that there exists a $\z \in \CH(S)$ such that the line drawn from $\z$ to $\y$ intersects $F$. $S$ is updated with the vertices of the neighboring $d$-simplex and the process repeats until $\y$ is contained in $\CH(S)$. Such a walk is not unique, but provably acyclic so it is guaranteed to converge \cite{edelsbrunner1989acyclicity}. We summarize this procedure in Algorithm \ref{alg:delaunaysparse} and refer the reader to the original paper for the specific subroutines required for each step.

\begin{algorithm}[htbp]
\caption{DelaunaySparse (Simplified)}\label{alg:delaunaysparse}
\begin{algorithmic}[1]
\Require $\y \in \CH(\X)$
\State  Generate an initial $d$-simplex $S$ to seed the iterative process
\While{$\y \not \in S$}
        \State Select a face $F$ of $S$ from which $\y$ is visible
        \State Complete the face $F$ into a new $d$-simplex $S^*$ 
        \State Update $S \leftarrow S^*$
\EndWhile
\end{algorithmic}
\end{algorithm}

\section{Analysis of Locality Regularization}\label{sec:theory}

\subsection{Locality As An Objective}
We first review the theory of locality when used as an objective, as in \eqref{eqn:exact}, to obtain local sparse representations. Locality was proposed in \cite{tasissa2023k} in the context of structured dictionary learning, and we recall some of their results.

\begin{lemma}\label{lem:circumdist}
Let $\X$ be in general position.  Let $I(S)$ denote a set of indices such that $\{ \x_j \ 
| \ j \in I(S)\}$ are the vertices of a $d$-simplex $S$ of $\DT(\X)$. Let $R$ and $\cb$ denote the radius and center of the circumscribing hypersphere intersecting the vertices. Then $\Vert \x_j - \cb \Vert = R$ for $j \in I(S)$ and $\Vert \x_j - \cb \Vert > R$ for $j \notin I(S)$.
\end{lemma}
\begin{proof}
    Note $\Vert \x_j - \cb \Vert = R$ for $j \in I(S)$ follows by definition of the circumscribing hypersphere. Any $d$-simplex of $\DT(\X)$ contains no other points of $\X$, which together with $\X$ being in general position implies that $\Vert \x_j - \cb \Vert > R$ for $j \notin I(S)$ as each of these $\x_j$ are strictly outside the circumscribing hypersphere. 
\end{proof}

\noindent We next recall a translation-invariance result for the locality regularizer under linear constraints. 

\begin{lemma}[Lemma 2 restated from \cite{tasissa2023k}]\label{lem:reg-relation}
    Let $\yt \in \CH(\X)$ and let $\w\in \Delta^n$ be such that $\X\w = \yt$. Let $\ell_{\yt}(\w) := \sum_{i=1}^n w_i \Vert \x_i - \yt \Vert^2$. Then for any $\y \in \R^d$, 
    \begin{equation*}
        \ell_{\yt}(\w) = \ell_\y(\w) - \Vert \y - \Tilde{\y} \Vert^2.
    \end{equation*}
\end{lemma}
\begin{proof}We compute as follows, noting that $\X \w=\tilde{\y}$ in the fifth line:
    \begin{align*}
        \ell_{\Tilde{\y}}(\w) &= \sum_{i=1}^n w_i \Vert \x_i - \Tilde{\y} \Vert^2 \\
        &= \sum_{i=1}^n w_i \Vert \x_i - \y + \y - \Tilde{\y} \Vert^2 \\
        &=  \sum_{i=1}^n w_i \Vert \x_i - \y \Vert^2 + \sum_{i=1}^n w_i \Vert  \y - \Tilde{\y} \Vert^2 + 2 \sum_{i=1}^n w_i \langle \x_i - \y, \y - \Tilde{\y} \rangle  \\
        &= \ell_\y(\w) + \Vert \y - \Tilde{\y} \Vert^2 +2 \left\langle \sum_{i=1}^n w_i \x_i - \y, \y - \Tilde{\y} \right\rangle  \\
        &= \ell_\y(\w) + \Vert \y - \Tilde{\y} \Vert^2 +2 \left\langle \Tilde{\y} - \y, \y - \Tilde{\y} \right\rangle \\
        &= \ell_\y(\w) + \Vert \y - \Tilde{\y} \Vert^2 - 2 \Vert \y - \Tilde{\y} \Vert^2\\
        &= \ell_\y(\w) - \Vert \y - \Tilde{\y} \Vert^2.
    \end{align*}
\end{proof}

The following result asserts structural properties of solutions to the exact problem \eqref{eqn:exact}.  It slightly generalizes a result of \cite{tasissa2023k} by considering $\y\in\CH(\X)$ that lie on the boundary of $d$-simplices of $\DT(\X)$, not just in the interior.

\begin{theorem}[Generalization of Theorem 2 from \cite{tasissa2023k}]\label{thm:TSP_2023_E}
   Let $\X$ be in general position.  Let $\y \in \CH(\X)$.  Then solving \eqref{eqn:exact} yields a solution $\w^*$ that is $(d+1)$-sparse. In particular, these nonzero entries correspond with the $k$-face, belonging to at least one $d$-simplex of $\DT(\X)$, containing $\y$ where $k \leq d$. 
\end{theorem}

\begin{proof}
    Let $S_1, S_2, \ldots, S_N$ denote the $d$-simplices of $\DT(\X)$ that include $\y$.  Let $F = \bigcap_{i=1}^N S_i$ denote the vertices of the $k$-face containing $\y$ common to these $N$ simplices (when $N=1$ the face containing $\y$ is the $d$-simplex itself). Let $\w^* \in \Delta^n$ denote the unique representation of $\y$ supported on $F$. Let $I(S_i)$ denote the indices in $\X$ of the elements of $S_i$ for each $i=1,\ldots, N$. Let $R_i$ denote the radius of the hypersphere circumscribing $S_i$ and let $\cb_i$ denote the center of theses hyperspheres. 
    
    For each $i$ suppose that $\w^i$ is another feasible solution supported on vertices $S_i'$ with indices $I(S_i')$. Suppose that $I(S_i') \cap I(S_i)^C \neq \emptyset$ so that $S_i'$ contains at least one vertex not in $S_i$. Noting that $\X\w^i=\y$ and using Lemma \ref{lem:reg-relation} we can write 
    \begin{align*}
     \ell_\y(\w^i) &= \sum_{j \in I(S_i')} w^i_j \Vert \y - \x_j \Vert^2\\
       &= \sum_{j \in I(S_i')} w^i_j \Vert \cb_i - \x_j \Vert^2 - \Vert \y - \cb_i \Vert^2.\end{align*}
       Using Lemma \ref{lem:circumdist} and in particular the fact that there exists at least one index $\alpha_i$ of $I(S_i')$ not in $I(S_i)$ so that $\Vert \cb_i - \x_{\alpha_i} \Vert > R_i$, it follows that
   \begin{align*}\ell_{\y}(\w^i)&> \sum_{j \in I(S_i')} w^i_j R_i^2 - \Vert \y - \cb_i \Vert^2 \\
       &= R_i^2 - \Vert \y - \cb_i \Vert^2\\
       &=\sum_{j \in I(S_i)} w^*_j \Vert \cb_i - \x_j \Vert^2 - \Vert \y - \cb_i \Vert^2\\
       &= \sum_{j \in I(S_i)}w^*_j \Vert \y - \x_j \Vert^2\\
       &= \sum_{j \in I(F)}w^*_j \Vert \y - \x_j \Vert^2\\
       &= \ell_\y(\w^*).
    \end{align*}
  Since this is true for all $i=1,\ldots, N$, any feasible solution not supported strictly on $F$ is suboptimal. 
\end{proof}

\begin{remark}
We note that if $\y$ lies in the interior of a $d$-simplex of $\DT(\X)$, then $N$=1 in the above proof.  
Otherwise $\y$ either lies on the interior of a $(d-1)$-face intersecting the boundary of $\CH(\X)$ and $N=1$, or $\y$ lies in a $k$-face with $k<d$ forming the intersection between $N\geq 2$ faces.
\end{remark}

\subsection[Analysis of Solutions to (R)]{Analysis of Solutions to \eqref{eqn:relaxed}}

Our aim is to prove a version of Theorem \ref{thm:TSP_2023_E} for the relaxed problem (\ref{eqn:relaxed}) which will depend on $\rho$.  We first establish a few technical preliminary results before proving Theorem \ref{thm:main_R}.  

In Lemma \ref{lem:rho_bound}, we establish that solutions to (\ref{eqn:relaxed}) must reconstruct $\y$ well when $\rho$ is small.
\begin{lemma}\label{lem:rho_bound}
    Fix $\rho > 0$.  Let $\y$ lie in a $d$-simplex of $\DT(\X)$ and let $\w_\rho$ denote a solution to \eqref{eqn:relaxed}.  Then $\Vert \X \w_\rho - \y\Vert^2 \leq \rho C,$ where $C$ is a constant depending on $\X$ and $\y$. 
\end{lemma}
\begin{proof}
    First, let $\w_e$ denote the solution to \eqref{eqn:exact}. We can obtain a bound on $\Vert \X \w_\rho - \y\Vert^2$ by noting that $\w_e$ is a suboptimal solution to \eqref{eqn:relaxed} but satisfies $\Vert\X \w_e - \y\Vert^{2} = 0$:
    \begin{align*}
        \Vert \X \w_\rho -\y \Vert^2 + \rho \sum_{i=1}^n (w_\rho)_i \Vert \x_i - \y\Vert^2 &\leq \rho \sum_{i=1}^n (w_e)_i \Vert \x_i - \y\Vert^2\\
        \Rightarrow\Vert \X \w_\rho - \y \Vert^2 &\leq \rho \sum_{i=1}^n (\w_e - \w_\rho)_i \Vert \x_i - \y\Vert^2 \\
        &\leq \rho (\max_{i} \Vert \x_i - \y\Vert^2 - \min_{i} \Vert \x_i - \y\Vert^2),
    \end{align*}and the result follows taking $C:=\displaystyle\max_{i} \Vert \x_i - \y\Vert^2 - \min_{i} \Vert \x_i - \y\Vert^2$.
\end{proof}

Next, we combine Theorem \ref{thm:TSP_2023_E} with Lemma \ref{lem:reg-relation}.

\begin{lemma}\label{lem:simplex_support}
    Let $\X$ lie in general position. 
 Let $\yt$ lie in the interior of a $d$-simplex of $\DT(\X)$. Let $\w_e$ be chosen as the solution of $\eqref{eqn:exact}$ so that $\yt = \X \w_e$. Let $\w \in \Delta^n$ be chosen so that $\X \w = \yt$ and there exists an $i$ such that $[\w_e]_i = 0$ and $\w_i > 0$ (i.e., $\w$ has a nonzero entry coinciding with a point in $\X$ that is not a vertex of the $d$-simplex of $\DT(\X)$ containing $\yt$). Then for any other $\y \in \R^d$, $\ell_\y(\w_e) < \ell_\y(\w).$
\end{lemma}

\begin{proof}
    From Theorem \ref{thm:TSP_2023_E}, we have that $\ell_{\yt}(\w_e) < \ell_{\yt}(\w)$.
      Now, using Lemma \ref{lem:reg-relation} we have that
    \begin{align*}
        \ell_{\y}(\w_e)  - \Vert \y - \Tilde{\y} \Vert^2 &< \ell_{\y}(\w)  - \Vert \y - \Tilde{\y} \Vert^2 \\
        \Rightarrow \ell_{\y}(\w_e) &< \ell_{\y}(\w).
    \end{align*}
\end{proof}





Finally, we state and prove an analogue to Theorem \ref{thm:TSP_2023_E} for the relaxed problem \eqref{eqn:relaxed}.

\begin{theorem}\label{thm:main_R}
    Let $\X$ lie in general position. Let $\y$ lie in $\CH(\X)$ and in the interior of a $d$-simplex of $\DT(\X)$ with vertices $S$. Let $\rho < \frac{d_{S\y}}{C}$, where $d_{S\y} = \min_{\z \in \partial S} {\Vert \z - \y \Vert^2}$ and $C$ is as in Lemma \ref{lem:rho_bound}. Then, solving \eqref{eqn:relaxed} yields a solution $\w_\rho$ that is $(d+1)$-sparse and whose nonzero entries correspond to $S$.
\end{theorem}

\begin{proof}
    By Lemma \ref{lem:rho_bound}, we have $\Vert \X \w_\rho - \y \Vert^2 \leq \rho C$ where $C$ depends only on $\X$ and $\y$.  Because $\y$ is contained in $\CH(S)$, it follows that for $\rho < \frac{d_{Sy}}{C}$, $\X \w_\rho $ is contained within $\CH(S)$. Suppose, towards a contradiction, that $\w_\rho$ has support not strictly on $I(S)$, the indices of $S$ in $\X$. 
    Since $\X \w_\rho$ is contained within the Delaunay simplex $\CH(S)$, Theorem \ref{thm:TSP_2023_E} tells us that there exists weights $\w_e$ that are supported only on $S$ and that $ \yt:= \X \w_\rho = \X \w_e$. Thus, $\frac{1}{2}\Vert \X \w_\rho - \y \Vert^2 = \frac{1}{2}\Vert \X \w_e - \y \Vert^2$. Moreover, as $\w_e$ solves $\eqref{eqn:exact}$ for the point $\yt$, we have that $\ell_{\yt}(\w_e) < \ell_{\yt}(\w_\rho)$. Lemma \ref{lem:simplex_support} then allows us to conclude that $\ell_\y(\w_e) < \ell_\y(\w_\rho)$. Now, comparing $\w_e$ to $\w_\rho$ in the objective of \eqref{eqn:relaxed} we have that:
\begin{align*}
        \frac{1}{2}\Vert \X\w_e - \y\Vert^2 + \rho  \sum_{i=1}^n [w_e]_i \Vert \x_i - \y \Vert^2 &< \frac{1}{2}\Vert \X\w_\rho - \y\Vert^2 + \rho  \sum_{i=1}^n [w_\rho]_i \Vert \x_i - \y \Vert^2
    \end{align*}
    which contradicts the assumption that $\w_\rho$ solves $\eqref{eqn:relaxed}$. 
 Thus, the solution must be supported strictly on $S$.
\end{proof}

\begin{remark}
We note that an important observation in Theorem \ref{thm:main_R} is that for  $\X$ in general position and for arbitrary $\y\in\mathbb{R}^{d}$ and $\rho > 0$, the solution $\w_\rho$ to (\ref{eqn:relaxed}) is identical to the solution to \eqref{eqn:exact} with $\X \w_\rho$ as the point to be represented.  This holds for arbitrary $\rho>0$ and $\y$, but $\X\w_{\rho}$ is only close to $\y$ when $\rho$ is small and $\y\in\CH(\X)$.
\end{remark}



\subsubsection[Reconstruction Outside CH(X)]{Reconstruction Outside $\CH(\X)$}
We turn our attention to the case where $\y \not \in \CH(\X)$. 
\begin{lemma}\label{lem:OutsideCH}
    Let $\w_\rho$ be the solution to \eqref{eqn:relaxed} and $\wpr=\displaystyle\argmin_{\w \in \Delta^n} \Vert \X\w - \y\Vert^{2}$, the weights that form the projection of $\y$ onto $\CH(\X)$. Then 
    \begin{equation*}
       \Vert \X \wpr - \y \Vert^2 \leq \Vert \X \w_\rho - \y \Vert^2 \leq \Vert \X \wpr - \y \Vert^2 + \rho C,
    \end{equation*}with $C$ as in Lemma \ref{lem:rho_bound}.
\end{lemma}

\begin{proof}
   The lower bound come from the definition of $\w_{proj}$. The upper bound is the same argument as Lemma \ref{lem:rho_bound} with the same constant $C$, but now the reconstruction error term of the projection does not vanish as $\X \w_{proj} \neq \y$.
\end{proof}

\begin{lemma}\label{lem:out_bound}
    For any $\rho > 0$, $\Vert \X \wpr - \X \w_\rho \Vert^2 \leq \rho C$ with $C$ as in Lemma \ref{lem:rho_bound}.
\end{lemma}

\begin{proof}
We recall \cite{beck2014introduction} that the projection $\X\w_{proj}$ of a point $\y$ onto the convex hull $\CH(\X)$ satisfies the condition 
\begin{equation*}
\langle \y - \X \w_{proj}, \textbf{v} - \X \w_{proj}\rangle \le 0, \text{ for all } \textbf{v}\in\CH(\X).
\end{equation*}
Now, we compute 
\begin{align*}
&\Vert\X\w_{\rho}-\y\Vert^{2}\\
=&\Vert\X\w_{\rho}-\X\w_{proj}\Vert^{2}+\Vert\X\w_{proj}-\y\Vert^{2}+2\langle \X\w_{\rho}-\X\w_{proj}, \X\w_{proj}-\y\rangle\\
\ge& \Vert\X\w_{\rho}-\X\w_{proj}\Vert^{2}+\Vert\X\w_{proj}-\y\Vert^{2}.
\end{align*}
This implies by Lemma \ref{lem:OutsideCH} that $\Vert\X\w_{\rho}-\X\w_{proj}\Vert^{2}\le \Vert\X\w_{\rho}-\y\Vert^{2}-\Vert\X\w_{proj}-\y\Vert^{2}\le\rho C$.

\end{proof}

\begin{proposition}
    Let $\X$ lie in general position. Let $\y \not \in \CH(\X)$ be such that its projection $\y_{proj}$ is on the interior of the $(d-1)$-face of the $d$-simplex containing $\y_{proj}$. Then for small enough $\rho$ the solution $\w_\rho$ is $(d+1)$-sparse and has nonzero entries corresponding to the vertices of the $d$-simplex in $\DT(\X)$ containing $\y_{proj}$.
\end{proposition}

\begin{proof}
    Let $F$ denote the $d$ vertices of the $(d-1)$-face containing $\y_{proj}$ and let $S$ denote the $d+1$ vertices of the $d$-simplex with $F$ as vertices. Note that the $(d-1)$-face identified by $F$ intersects with no other $d$-simplex of $\DT(\X)$ since it lies on the boundary of $\CH(\X)$. Let $\partial F$ denote the boundary of $\CH(F)$, and pick $\rho < \frac{1}{C}\min_{\z \in \partial F} \Vert \z - \y_{proj} \Vert^2$ to ensure that $\X \w_\rho$ lies within the $d$-simplex containing $\y_{proj}$ via Lemma \ref{lem:out_bound}. We note that our choice of $\rho$ implies that $\X \w_{\rho}$ lies either on the face represented by $F$ or on the interior of the $d$-simplex of $\DT(\X)$ containing $F$. Suppose towards a contradiction that there exists an $i \not \in I(S)$ such that $[w_\rho]_i > 0$. Since $\X \w_\rho$ is contained within $\CH(S)$, Theorem \ref{thm:TSP_2023_E} tells us that there exists weights $\w_e$ that are supported only on $S$ and that $ \yt:= \X \w_\rho = \X \w_e$. Thus, $\frac{1}{2}\Vert \X \w_\rho - \y \Vert^2 = \frac{1}{2}\Vert \X \w_e - \y \Vert^2$. Moreover, as $\w_e$ solves $\eqref{eqn:exact}$ for the point $\yt$, we have that $\ell_{\yt}(\w_e) < \ell_{\yt}(\w_\rho)$. Lemma \ref{lem:simplex_support} then allows us to conclude that $\ell_\y(\w_e) < \ell_\y(\w_\rho)$. Now, comparing $\w_e$ to $\w_\rho$ in the objective of \eqref{eqn:relaxed} we have that:
\begin{align*}
        \frac{1}{2}\Vert \X\w_e - \y\Vert^2 + \rho  \sum_{i=1}^n [w_e]_i \Vert \x_i - \y \Vert^2 &< \frac{1}{2}\Vert \X\w_\rho - \y\Vert^2 + \rho  \sum_{i=1}^n [w_\rho]_i \Vert \x_i - \y \Vert^2
    \end{align*}
    which contradicts the assumption that $\w_\rho$ solves $\eqref{eqn:relaxed}$. 
 Thus, the solution must be supported strictly on $S$.
\end{proof}

\subsubsection[X Not In General Position]{$\X$ Not In General Position}
When the points of $\X$ are not in general position there exists a subset of $m>d+1$ points that all lie on the surface of a hypersphere with radius $R$.

\begin{proposition}\label{prop:nonunique}
    Let $\X$ be a set of $m>d+1$ points that lie on the surface of a hypersphere with radius $R$. Let $\y \in \CH(\X)$. Then there exists a constant $K$ such that for any $\w \in \Delta^m$ with $\X\w = \y$, $\displaystyle\sum_{i=1}^m w_i \Vert \x_i - \y \Vert^2 = K$.
\end{proposition}

\begin{proof}
    Using Lemma \ref{lem:reg-relation} we have for any feasible $\w$ and $\cb$ as the center of the circumscribing hypersphere of radius $R$ that
    \begin{align*}
        \sum_{i=1}^m w_i \Vert \x_i - \y \Vert^2 &= \sum_{i=1}^m w_i \Vert \x_i - \cb \Vert^2 - \Vert \cb - \y \Vert^2\\
        &=\sum_{i=1}^m w_i R^2 - \Vert \cb - \y \Vert^2\\
        &= R^2 - \Vert \cb - \y \Vert^2\\
       & =:K,
    \end{align*}
    which is constant for any feasible $\w$. 
\end{proof}

\begin{proposition}\label{prop:nonunique_sol}
    Let $\X$ not lie in general position. Then there exists a circumscribing hypersphere that intersects more than $d+1$ points of $\X$ and does not contain any other points of $\X$ in the interior. 
 Denote these intersected points as $S$. Let $\y \in \CH(S)$, and let $S_1, S_2, \ldots, S_N$ denote the subsets of $S$ with cardinality $d+1$ such that $\y \in \CH(S_i)$ for $i=1, \ldots, N$ for  $N \geq 1$. Let $\w^1, \w^2, \ldots, \w^N$ be such that $\X \w^i = \y$ for all $i=1,\ldots, N$. Then for any $\w \in \Delta^n$ such that $\X \w = \y$ with support on at least one point of $\X$ not in $S$,
    \begin{equation*}
        \ell_\y(\w^1) = \ell_\y(\w^2) = \dots = \ell_\y(\w^N) < \ell_\y(\w).
    \end{equation*}
\end{proposition}

\begin{proof}
    Proposition \ref{prop:nonunique} shows that $\ell_\y(\w^1) = \ell_\y(\w^2) = \dots = \ell_\y(\w^N)$ as each $\w^i$ lies on the same circumscribing hypersphere. Let $\w \in \Delta^n$ with support on at least one point of $\X$ not in $S$; let $S'$ denote the support of $\w$. Let $\cb$ denote the center of the circumscribing hypersphere of $S$ with radius $R$. Noting that $\X\w=\y$ and using Lemma \ref{lem:reg-relation} we can write, 
    \begin{align*}
     \ell_\y(\w) &= \sum_{j \in I(S')} w_j \Vert \y - \x_j \Vert^2\\
       &= \sum_{j \in I(S')} w_j \Vert \cb - \x_j \Vert^2 - \Vert \y - \cb \Vert^2.\end{align*}
       Using the fact that there exists at least one index $\alpha$ of $I(S')$ not in $I(S)$ so that $\Vert \cb - \x_{\alpha} \Vert > R$, it follows that for any $\w^i$ and $S_i$,
   \begin{align*}\ell_{\y}(\w)&> \sum_{j \in I(S')} w_j R^2 - \Vert \y - \cb \Vert^2 \\
       &= R^2 - \Vert \y - \cb \Vert^2\\
       &=\sum_{j \in I(S_i)} w^i_j \Vert \cb - \x_j \Vert^2 - \Vert \y - \cb \Vert^2\\
       &= \sum_{j \in I(S_i)}w^i_j \Vert \y - \x_j \Vert^2\\
       &= \ell_\y(\w^i).
    \end{align*}
\end{proof}

\begin{corollary}\label{cor:superposition}
    Under the same conditions as Proposition \ref{prop:nonunique_sol}, the solution to \eqref{eqn:exact} is nonunique and can be represented as $\sum_{i=1}^N \beta_i \w^i$ for any $\beta \in \Delta^N$.
\end{corollary}

\begin{proof}
    Proposition \ref{prop:nonunique_sol} establishes that each $\w^i$ has optimal locality subject to $\X \w = \y$. Let $\beta \in \Delta^N$, then we can write
    \begin{equation*}
        \ell_{\y}\left( \sum_{i=1}^N \beta_j \w^i\right) = \sum_{j=1}^n \sum_{i=1}^N \beta_i w^i_j \Vert \x_j - \y \Vert^2 = \sum_{i=1}^N \beta_i\sum_{j=1}^n w^i_j \Vert \x_j - \y \Vert^2 = \sum_{i=1}^N \beta_i \ell_{\y}(\w^i)
    \end{equation*}
    which is equal to the same constant as each $\ell_\y (\w^i)$ since $\ell_\y(\w^1) = \ell_\y(\w^2) = \dots = \ell_\y(\w^N)$. 
\end{proof}

\begin{remark}
    A similar result to Corollary \ref{cor:superposition} can be shown for \eqref{eqn:relaxed} under the conditions of Theorem \ref{thm:main_R}, but with $\X$ not lying in general position and $\y$ with the same conditions as Proposition \ref{prop:nonunique_sol}. In this case the argument of Corollary \ref{cor:superposition} can be mimicked with $\X \w_\rho$ in place of $\y$. 
\end{remark}

\subsection{Stability Analysis}

In practical scenarios, data tends to be noisy and susceptible to outliers. This section studies the stability of optimal sparse solutions for \eqref{eqn:relaxed} when subjected to additive noise in the input data. We specifically study the setting in which the initial data $\y$ is perturbed by additive noise, yielding $\tilde{\y}$, and we make the assumption that $\Vert\y-\tilde{\y}\Vert\le \epsilon$. Our stability analysis relies on the concept of a local dictionary, originally introduced in \cite{tasissa2023k}, which we restate below.

\begin{definition}
For a set of points 
$\X$ in general position with a unique Delaunay triangulation $\text{DT}(\X)$, let $\y\in \R^{d}$ be an interior point of a $d$-simplex of $\text{DT}(\X)$. Let $I(S)$ denote a set of indices such that $\{ \x_j \ 
| \ j \in I(S)\}$ are the vertices of a $d$-simplex $S$ of $\DT(\X)$ that contains $\y$. We denote $\X_{I(S)}\in \R^{d\times (d+1)}$ as the \unemph{local dictionary} corresponding to $\y$.
\end{definition}
Using the above definition of a local dictionary, we define barycentric coordinates as follows.
\begin{definition}
    For a local dictionary $\X_{I(S)} \in \R^{d\times (d+1)}$ and a data point $\y\in \R^{d}$, the \unemph{barycentric coordinates} of $\y$ are the unique solution $\x$ to the linear system $\mathbf{B}_{I(S)}\mathbf{x} = \mathbf{z}$,
where $\mathbf{B}_{I(S)}\in \R^{(d+1)\times (d+1)}$ is defined as $\mathbf{B}_{I(S)}= \begin{pmatrix}\X_{I(S)}\\ \one_{d} \end{pmatrix}$ and $\z = \begin{pmatrix}\y \\ 1\end{pmatrix}$. 
\end{definition}

We achieve stability as follows: 

\begin{figure}
    \centering    \includegraphics{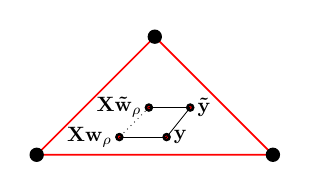}
    \caption{The point $\y$ is perturbed to $\tilde{\y}$. Our analysis relies on estimating the size of the dashed line using triangle inequality given the size of the solid lines.}
    \label{fig:enter-label}
\end{figure}

\begin{theorem}\label{thm:stability_theorem}
 Let $\X$ lie in general position. 
 Let $\y$ and $\tilde{\y}$ lie in $\CH(\X)$ and be interior points of the same $d$-simplex of $\DT(\X)$ with $\Vert\y-\tilde{\y}\Vert\le \epsilon$. Denote the vertices of this Delaunay simplex as $S$.  Let $\rho< \rho^*:= \min\left(\frac{d_{S_{\y}}}{C_{y}},\frac{d_{S_{\tilde{\y}}}}{C_{\tilde{y}}}\right),$
 where $d_{S\y} = \min_{\z \in \partial S} {\Vert \z - \y \Vert^2}$ and $C_{\y}, C_{\tilde{\y}}$ are as in Lemma \ref{lem:rho_bound}. Let $\w_\rho$ be the solution to \eqref{eqn:relaxed} using the true data point $\y$ and let $\tilde{\w}_{\rho}$ be the solution to \eqref{eqn:relaxed} using the noisy data point $\tilde{\y}$. Then we have:
 \[
\Vert\w_{\rho}-\tilde{\w}_{\rho}\Vert\le \frac{1}{\sigma_{\min}(\mathbf{B}_{I(S)})}\,(\sqrt{\rho}(\sqrt{C_{\y}}+\sqrt{C_{\tilde{\y}}})+\epsilon),
 \]
 where $\mathbf{B}_{I(S)}= \begin{pmatrix}\X_{I(S)}\\ \one_{d} \end{pmatrix}$ and $\X_{I(S)}$ is the local dictionary corresponding to $\tilde{\y}$ and where $\sigma_{\min}$ is the smallest singular value of a matrix.
\end{theorem}

\begin{proof}
With $\rho<\rho^*$, we apply Theorem \ref{thm:main_R} to $\y$ and $\tilde{\y}$. It follows that both $\w_{\rho}$ and $\tilde{\w}_{\rho}$ are $(d+1)$-sparse and supported on $S$. In what follows our first aim is to bound $\Vert\X\w_{\rho}-\X\tilde{\w}_\rho\Vert$. To do so, we start by applying Lemma \ref{lem:rho_bound} on $\y$ and $\tilde{\y}$ to obtain the following bounds:
\begin{align*}
\Vert\X\w_{\rho}-\y\Vert\le \sqrt{\rho} \sqrt{C_{\y}},\\
\Vert \X\tilde{\w}_{\rho}-\tilde{\y}\Vert \le \sqrt{\rho} \sqrt{C_{\tilde{\y}}}.
\end{align*}
Using the above bounds and the assumption that $\Vert\y-\tilde{\y}\Vert\le \epsilon$, we bound $\Vert\X\w_{\rho}-\X\tilde{\w}_\rho\Vert$ as follows:
\begin{align}
\Vert\X\w_{\rho}-\X\tilde{\w}_\rho\Vert  &= 
\Vert(\X\w_{\rho}-\y)+(\y-\tilde{\y})+(\tilde{\y}-\X\tilde{\w}_\rho)\Vert \nonumber\\ 
&\le \Vert\X\w_{\rho}-\y\Vert+\Vert\y-\tilde{\y}\Vert+\Vert\tilde{\y}-\X\tilde{\w}_\rho\Vert \nonumber\\ 
&\le \sqrt{\rho}(\sqrt{C_{\y}}+\sqrt{C_{\tilde{\y}}})+\epsilon. \label{eq:bound_xw_xwn}
\end{align}
Note that $\X\w_{\rho}$ and $\X\tilde{\w}_{\rho}$ are in the same $d$-simplex $S$ of $\text{DT}(\X)$. It follows that they have the same local dictionary denoted by $\X_{I(S)}$. Using this and the fact that $\w_{\rho}$ and $\tilde{\w}_{\rho}$ are $(d+1)$-sparse and supported on $S$, we have $\X\w_{\rho}= \X_{I(S)}\z^*$ and $\X\tilde{\w}_{\rho}= \X_{I(S)}\z$ where $\z^*,\z$ in $\Delta^{d+1}$ respectively are $\w_{\rho}$ and $\tilde{\w}_{\rho}$ restricted to $I(S)$. Define $\mathbf{q},\tilde{\mathbf{q}}\in \R^{d+1}$ as follows: $\mathbf{q} = \begin{pmatrix} \X\w_{\rho}\\1 \end{pmatrix}$ and $\tilde{\mathbf{q}} = \begin{pmatrix} \X\tilde{\w}_{\rho}\\1 \end{pmatrix}$. Note that $\Vert\mathbf{q}-\tilde{\mathbf{q}}\Vert=\Vert\X\w_{\rho}-\X\tilde{\w}_\rho\Vert $. Further, $\mathbf{q} = \mathbf{B}_{I(S)}\z^*$ and $\tilde{\mathbf{q}}=\mathbf{B}_{I(S)}\z$. We now lower bound  $\Vert\mathbf{q}-\tilde{\mathbf{q}}\Vert$ as follows:
\begin{equation*}
\Vert\mathbf{q}-\tilde{\mathbf{q}}\Vert = \Vert\mathbf{B}_{I(S)}(\z^*-\z)\Vert\ge \sigma_{\min}(\mathbf{B}_{I(S)})\Vert\z^*-\z\Vert,\\
\end{equation*}
\noindent where $\sigma_{\min}(\mathbf{B}_{I(S)})>0$ because the vertices of the Delaunay simplex are in general position. Combining the above bound with the bound in \eqref{eq:bound_xw_xwn} yields
\begin{equation*}
\Vert\z^*-\z\Vert \le\frac{1}{\sigma_{\min}(\mathbf{B}_{I(S)})}\,(\sqrt{\rho}(\sqrt{C_{\y}}+\sqrt{C_{\tilde{\y}}})+\epsilon).\end{equation*}

\noindent The above equality establishes that the barycentric representations of $\X\w_{\rho}$ and $\X\tilde{\w}_{\rho}$ are bounded. Since $\w_{\rho}$ and $\tilde{\w}_{\rho}$ are $(d+1)$-sparse and supported on $S$, the proof concludes by noting that $\Vert\z^*-\z\Vert = \Vert\w_{\rho}-\tilde{\w}_{\rho}\Vert$. 

\end{proof}

\section{Experiments}\label{sec:experiments}
In each of the below experiments we solve \eqref{eqn:relaxed} using the CVXOPT library \cite{cvxopt} with a custom KKT system solver to take advantage of structure of the problem. Details of this custom component are provided in Appendix \ref{sec:app-comp}. 

\subsection{Comparison of Theoretical Bound}
 Theorem \ref{thm:main_R} provides a bound that guarantees that the vertices of the Delaunay simplex containing a point $\y$ will be identified via solving (\ref{eqn:relaxed}). In this experiment we compare how well the theory matches reality. To demonstrate this in a visual way we let $d=2$ and $n=10$. $\X$ is chosen to have be uniform samples on the unit square. $\X$ is then centered about the mean and normalized to have maximum $\ell_2$ norm of $1$. We sample 10,000 points from the convex hull. We then solve \eqref{eqn:relaxed} for each point with $\rho=2^k$ as a way of quantizing the search space of $\rho$; here $k\in\{-32, -31, \ldots, 2\}$. We find the largest $k$ such that the true vertices of the Delaunay simplex are identified. We compare the results in Figure \ref{fig:2Drhobound}. We observe that in general $\rho$ can be much larger than the bound in Theorem \ref{thm:main_R} suggests. This is especially so for points that lie near the boundary of a simplex, as the theory suggests $\rho$ would need to be very small, while we observe much larger $\rho$ work. 

 \begin{figure}[t!]
     \centering
\includegraphics[width=\linewidth]{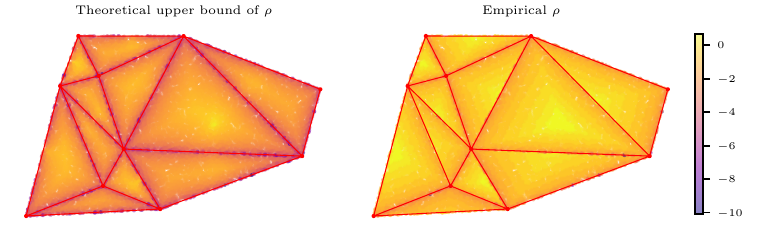}
     \caption{Sampled points within the convex hull colored by: \emph{left}: the bound from Theorem \ref{thm:main_R} and \emph{right}: the approximate $\rho$ needed empirically to determine the true vertices quantized to the nearest power of 2 (hence the distinct color bands). Points are colored according to $\log_{10}(\rho)$. In general the bound is pessimistic as to how small $\rho$ needs to be. In particular, we note that several triangles have points at the boundaries obtaining a true solution with relatively high $\rho$, where the theory suggests that points near the boundary of a triangle could need very small $\rho$.}
     \label{fig:2Drhobound}
 \end{figure}

\subsection{Solution Path}
We consider the optimal solutions of the locality regularized least squares problem as a function of $\rho$: 
\begin{align}\label{eqn:relaxed_solutions}
\w_{\rho}^* := \argmin_{\mathbf{w}\in \Delta^{n}} \frac{1}{2}\lVert \mathbf{X}\mathbf{w} - \mathbf{y}\rVert^2 + \rho \sum_{i=1}^n w_i \lVert \mathbf{x}_i - \mathbf{y} \rVert^2. 
\end{align}
We define the set $\{\w_{\rho}^*:\rho\in [0,\infty]\}$ to be the solution path. The solution paths for the standard and generalized LASSO problem have been studied and characterized in  \cite{osborne2000new,tibshirani2011solution}. In particular, these characterizations establish that the solution path is piecewise linear with respect to the regularization parameter. Various approaches, such as the least squares regression (LARS) method proposed in \cite{efron2004least} and the homotopy algorithms outlined in \cite{donoho2008fast,osborne2000new} utilize this structure to design efficient algorithms. 

In this paper, our focus is in exploring the solution paths of \eqref{eqn:relaxed_solutions} which falls under the category of regularized quadratic programs. It is generally known that generic classes of these programs give piecewise linear solution paths as functions of the regularization parameter 
 \cite{gu2017solution,gu2015new}. First we show an example in Figure \ref{fig:2Dsolutionpath} of the reconstructed point in $d=2$ as $\rho$ varies from choosing the nearest neighbor to exact representation in terms of the vertices of the Delaunay simplex containing the point. To view the solution path in higher dimensions we show an example with $d=5$ in Figure \ref{fig:highDsolutionpath}. We plot each weight as a function of $\rho$. Note that some weights become non-zero for a sub-interval, so one cannot conclude that the solution path adds one vertex of the Delaunay simplex at a time.

\begin{figure}[t!]
    \centering
    \includegraphics[width=\linewidth]{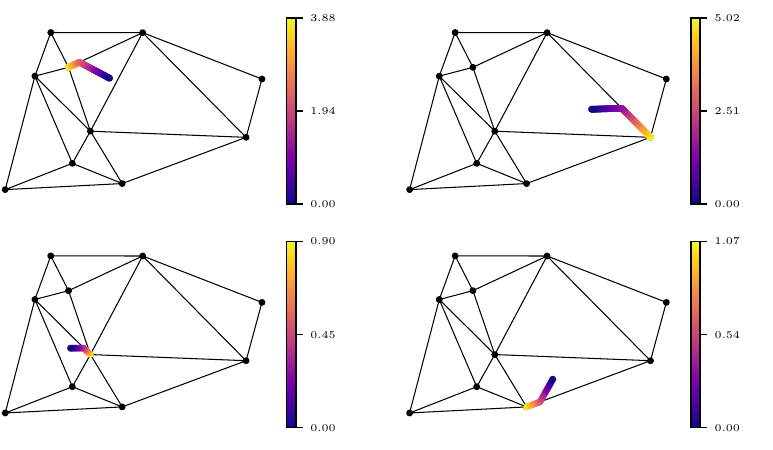}
    \caption{Each of these plots shows the representations formed by the solution path of $\w_\rho$ for a $d=2$ triangulation. Each path is colored as the representation progresses from the nearest neighbor (large $\rho$) to $\y$ (small $\rho)$.  As $\rho$ increases, the solution passes from being dense (3-sparse), to approximately 2-sparse to approximately 1-sparse.}
    \label{fig:2Dsolutionpath}
\end{figure}

\begin{figure}
    \centering
    \includegraphics[width=\linewidth]{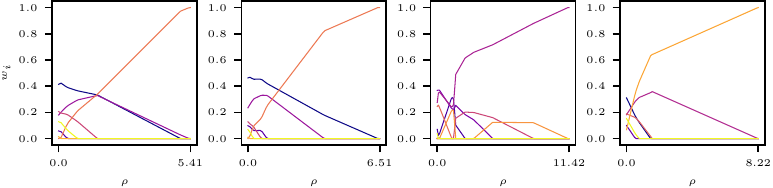}
    \caption{Solution paths of $\rho$: coefficient values versus $\rho$. In these plots, $\y$ was chosen in $\R^5$ and the coefficient values $w_i$ are plotted individually versus $\rho$. The $\rho$-axis is truncated to start on the right to only include a small portion of the regime where $\w$ has one nonzero entry on the nearest neighbor. Each plot shows the expected piecewise linear solution paths.  Moreover, we observe sparsity emerge with increasing $\rho$.}
    \label{fig:highDsolutionpath}
\end{figure}

\subsection{Comparison With Other Delaunay Triangle Identification Methods}
We provide an overview of comparisons between the four methods discussed within this paper that can be used to identify the $d$-simplex of $\DT(\X)$ containing a given $\y$. These comparisons are summarized in Table \ref{tab:comparison}.
\renewcommand{\arraystretch}{1.25}
\begin{table}[htbp]
\centering
\begin{tabular}{|c||c|c|c|c|c|} 
 \hline
 Method & Does Not Require Unique $\DT(\X)$ & Projects $\y \not \in \CH(\X)$  \\ 
 \hline\hline
 Convex Hull LP & \wmark & \wmark \\ 
 \hline
 DelaunaySparse & \cmark & \wmark \\
 \hline
 \eqref{eqn:exact} & \wmark & \wmark \\
 \hline
 \eqref{eqn:relaxed} & \wmark & \cmark   \\
 \hline
\end{tabular}
\caption{Applicability of each method for nonstandard situations.}
\label{tab:comparison}

\end{table}

\subsubsection{Non-Unique Delaunay Triangulation}
When $\X$ is not in general position then there are multiple Delaunay triangulations, as visualized in Figure \ref{fig:empty_circumscribing_hypersphere}. In this case $\y$ may lie in a circumscribing hypersphere containing $>d+1$ points. The locality based methods do not distinguish between possible solutions due to Proposition \ref{prop:nonunique_sol} and Corollary \ref{cor:superposition}. Similarly, the lifting map projects points on the boundary of same hypersphere to the same lower face of the convex hull of the lifted points. DelaunaySparse will provide an arbitrary valid $d$-simplex containing $\y$; for details see \cite{chang2020algorithm}.

\subsubsection{Handling Queries Outside The Convex Hull}
The Convex Hull LP becomes unbounded when $\y$ lies outside the convex hull and thus cannot provide a useful solution. DelaunaySparse does not provide for when $\y$ lies outside the convex hull, but can detect when it occurs. The code provided with their paper includes a subroutine to solve the projection problem. The exact locality problem (\ref{eqn:exact}) by definition can only handle points within the convex hull, but the relaxation (\ref{eqn:relaxed}) can provide solutions arbitrarily close to the projection as described in Lemma \ref{lem:out_bound}. 

\subsubsection{Empirical Scaling}
 To give a sense about how these methods' runtimes compare in practice, we demonstrate the scaling across $n$ and $d$ respectively. For each $n,d$ pair we generate $\X$ with each vector as a uniform sample from the unit $d$-hypercube. Then, we sample 50 points from $\CH(\X)$ by sampling $\w$ uniformly from $\Delta^n$ and forming $\y = \X \w$. For each sample we solve each of the methods and measure the runtimes; here Convex Hull LP and DelaunaySparse serve as the baselines. We solve Convex Hull LP, \eqref{eqn:exact}, and \eqref{eqn:relaxed} using CVXOPT in Python, and for $\eqref{eqn:relaxed}$ we fix $\rho=10^{-7}$. For DelaunaySparse we use code provided by the authors\footnote{\url{https://vtopt.github.io/DelaunaySparse/}}. We note that CVXOPT solves LPs and QPs approximately via an interior point solver, but thresholding the returned weights often matches the true vertices. We do not verify that the returned vertices are correct as we are focused on measuring the runtimes, which we found to not significantly change when tuning the threshold and $\rho$ for correct identification. Results are shown in Figure \ref{fig:scaling} with mean and standard deviation of the runtimes plotted.  We see that the all methods scale roughly the same in $n$. Convex Hull LP, \eqref{eqn:exact}, and \eqref{eqn:relaxed} all scale similarly in $d$ while DelaunaySparse scales noticeably worse. We note that, aside from DelaunaySparse, these methods are not optimized, but the scaling patterns should remain unaffected. We measured all results using a maximum of 8 cores from an Intel Xeon Gold 6248, which was provisioned by a high-performance computing cluster.

    \begin{figure}[htbp]
        \centering
        \includegraphics[width=\linewidth]{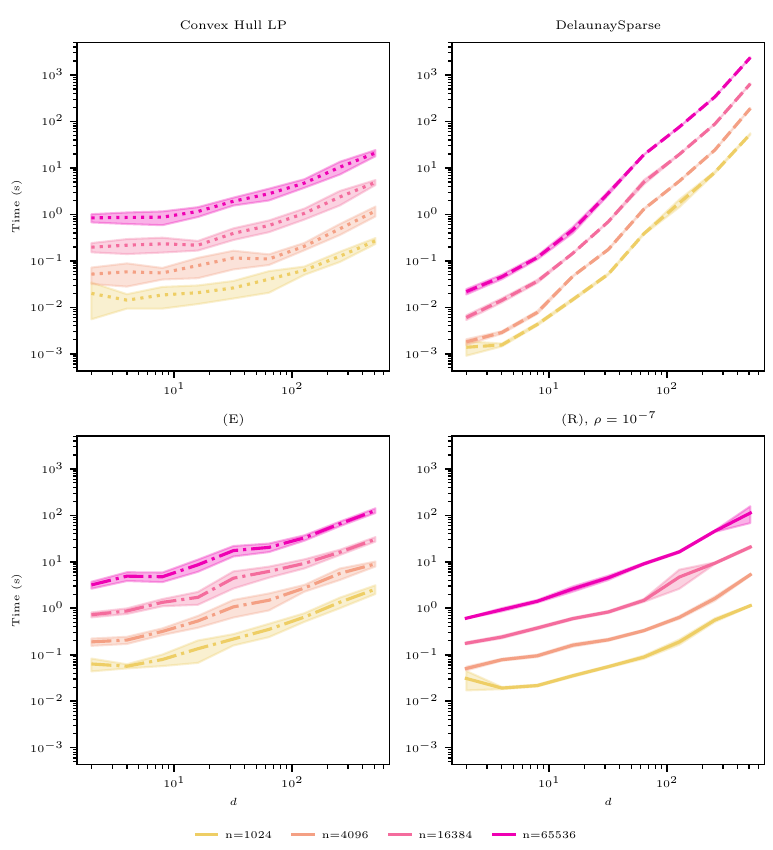}
        \caption{We show empirical scaling results comparing runtime (s) versus data dimension ($d$) for various $n$. We note that all methods scale approximately equally in $n$. Convex Hull LP, \eqref{eqn:exact}, and \eqref{eqn:relaxed} scale approximately equally in $d$. DelaunaySparse scales noticeably worse in $d$, which validates the theoretical curves.}
        \label{fig:scaling}
    \end{figure}

\subsubsection[Empirical Comparison of w\_ rho to w\_e]{Empirical Comparison of $\w_\rho$ to $\w_e$}

We demonstrate how $\w_\rho$, the solution to \eqref{eqn:relaxed}, compares to $\w_e$, the solution to \eqref{eqn:exact}, as $\rho$ changes. 
With $n$ fixed as $n=250$, we vary $d=[3, 9, 27, 81]$ and generate $\X$ with each vector as a uniform sample from the unit $d$-hypercube. For each $d$ we sample 50 points from $\CH(\X)$ by sampling $\w$ uniformly from $\Delta^n$ and forming $\y = \X \w$. Then we find $\w_e$ by solving \eqref{eqn:exact} using the simplex method and compare to $\w_\rho$ for $\rho=1.5^k$ for $k\in\{-32, -31,\ldots, 19\}$. For each sample we compute $\Vert \w_e -\w_\rho\Vert_1$ and measure the accuracy of the support identification by computing $\frac{\vert I(S)\cap I(S') \vert}{\vert I(S)\cup I(S')\vert}$ where $I(S)$ denotes the indices of the support of $\w_\rho$ and $I(S')$ denotes the indices of the support of $\w_e$. Results are visualized in Figure \ref{fig:rho_effect}.

\begin{figure}[htbp]
    \centering
    \includegraphics[width=\linewidth]{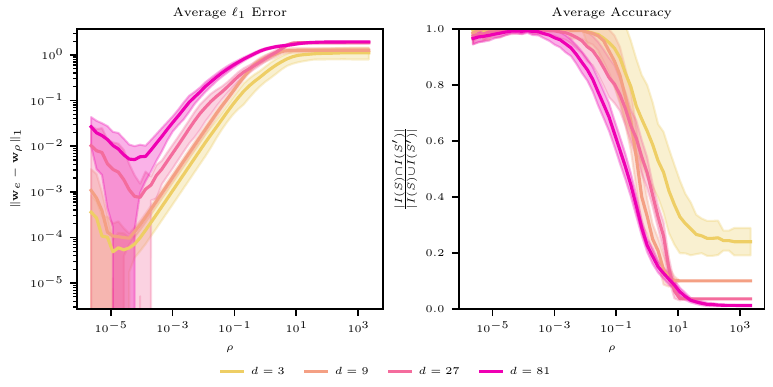}
    \caption{How $\w_\rho$ compares to $\w_e$ as $\rho$ varies. We note that numerical issues begin to arise as $\rho$ becomes small, causing deviation from the theory.}
    \label{fig:rho_effect}
\end{figure}

\section{Conclusion}
\label{sec:conclusion}

We analyzed a locality-regularized coding problem (\ref{eqn:relaxed}), establishing sparsity of its solutions under mild constraints.  Applications to the problem of identifying the $d$-simplex in $\DT(\X)$ that contains a specified point $\y$ were established and shown to be competitive with existing methods in terms of empirical runtime and scaling with dimension of data.  In particular, the proposed approach is useful even in the case that the observed data does not belong to the convex hull of the generating dictionary atoms.  

The proposed relaxed problem gives interpretable features that can be used for unsupervised and semisupervised learning.  Applications of these features (for both fixed and learned dictionaries) to high-dimensional image processing is a topic of future work.  Moreover, while the proposed schema is for data in $\mathbb{R}^{d}$, related notions of locality-regularized learning have been considered in the Wasserstein space \cite{mueller2023geometrically}.  Extending our results on sparsity to this context is an interesting future work, albeit one complicated by the curvature of the Wasserstein space.

\vspace{10pt}

\noindent\textbf{Acknowledgments:}  We are grateful for stimulating conversations with Demba Ba (Harvard), Tyler Chang (Argonne National Lab), Andrew Gillette  (Lawrence Livermore National Lab), and Matthew Hudes (Johns Hopkins). The authors acknowledge the Tufts University High Performance Compute Cluster\footnote{\url{https://it.tufts.edu/high-performance-computing}} which was utilized for the research reported in this paper.

\bibliographystyle{ieeetr}
\bibliography{bib}

\appendix
\section{Quadratic Program Computational Aspects}\label{sec:app-comp}

Recall that the main optimization problem we aim to solve is given by 
\begin{align*}
\argmin_{\w \in \Delta^n} \quad\frac{1}{2}\lVert \mathbf{X}\mathbf{w} - \mathbf{y}\rVert^2 + \rho \,\mathbf{w}^T\mathbf{c},
\end{align*}
where $\mathbf{c}\in \R^n$ and is defined as $c_i := \Vert \mathbf{x}_i - \mathbf{y}\Vert^2$.
This is a quadratic program and its standard form is: 
\begin{align}\label{eqn:QP2}
    \argmin_{\mathbf{w} \in \R^n} \quad \frac{1}{2} \mathbf{w}^T \X^T\X\mathbf{w}  + \mathbf{w}^T(\rho \mathbf{c} - \X^T\mathbf{y})\quad
    \text{s.t.} \quad\one_n^T \mathbf{w} = 1\text{ and }
    -\mathbf{I}_n \mathbf{w} \leq \mathbf{0}_n.
\end{align}
Note that, while the objectives of the two minimization programs differ, they both admit the same minimizers. The quadratic program (\ref{eqn:QP2}) can be solved effectively using an interior point solver. In our implementation, we use the CVXOPT library \cite{cvxopt}. We note that the dominant computational cost is solving the following KKT system at each iteration\footnote{\url{http://cvxopt.org/userguide/coneprog.html\#exploiting-structure}}
\begin{equation}
    \begin{bmatrix}
        \X^T\X & \one_n & -\mathbf{I}_n \\
        \one_n^T & 0 & 0\\
        -\mathbf{I}_n & \mathbf{0} & \mathbf{D}
    \end{bmatrix} 
    \begin{bmatrix}
        \mathbf{u_w} \\ u_y \\ \mathbf{u_z} 
    \end{bmatrix}
    =
    \begin{bmatrix}
        \mathbf{b_w} \\ b_y \\ \mathbf{b_z} 
    \end{bmatrix}
\end{equation}
In the above system, $\mathbf{D}$ is a diagonal matrix. In each iteration, given $\mathbf{D}$ and the right hand side vector $\mathbf{b}$, we solve for the vector $\mathbf{u}$. A direct solution of this problem requires solving an $O(n)$ linear system. This approach exhibits inefficient scaling since the cost of explicit construction of $\X^T\X$ is $O(n^2d)$. By avoiding the explicit formation of $\X^T\X$, we can employ manipulations on the linear system to achieve an indirect solution. In what follows, we detail the process.

Using the first and last block equations, we can equivalently write $\mathbf{u}_z$ as:
\begin{equation*}
    \mathbf{u}_z = \mathbf{D}^{-1}(\mathbf{u_w} + \mathbf{b_z}).
\end{equation*}
We can then reduce the linear system to two block equations: 
\begin{equation*}
    \begin{bmatrix}
        \X^T\X- \mathbf{D}^{-1}& \one_n  \\
        \one_n^T & 0 
    \end{bmatrix} 
    \begin{bmatrix}
        \mathbf{u_w} \\ u_y
    \end{bmatrix}
    =
    \begin{bmatrix}
        \mathbf{b_w} + \mathbf{D}^{-1}\mathbf{b_z} \\ b_y
    \end{bmatrix}
\end{equation*}
The first equation can be rewritten as a fixed point equation in terms of $\mathbf{u_x}$: 
\begin{equation*}
    \mathbf{u_w} = \mathbf{D}\left(\X^T\X\mathbf{u_w} + u_y \one_n - \mathbf{b_w} - \mathbf{D}^{-1}\mathbf{b_z}  \right)
\end{equation*}
We can then use the above form and in the second equation to obtain:
\begin{equation*}
    \one_n^T \left( \mathbf{D}\left(\X^T\X\mathbf{u_w} + u_y \one_n -\mathbf{b_w} - \mathbf{D}^{-1}\mathbf{b_z} \right)\right) = b_y
\end{equation*}
We can then solve for the scalar $u_y$ as follows:
\begin{equation*}
    u_y = \frac{b_y - \one_n^T\left(\mathbf{D}\mathbf{X}^T\mathbf{X}\mathbf{u_w} - \mathbf{D}\mathbf{b_w} - \mathbf{b_z}  \right)}{\Tr(\mathbf{D})}.
\end{equation*}
Now we can plug this expression for $u_y$ into the first equation to obtain:
\begin{equation*}
    \left( \X^T\X - \mathbf{D}^{-1} - \frac{\one_n \one_n^T\mathbf{D}\X^T\X}{\Tr(\mathbf{D})} \right) \mathbf{u_w} = \mathbf{b_w} + \mathbf{D}^{-1}\mathbf{b_z} - \frac{b_y + \one_n^T\left(\mathbf{D}\mathbf{b_w} + \mathbf{b_z}  \right)}{\Tr(\mathbf{D})} \one
\end{equation*}
We now introduce a change of variables $\X \mathbf{u_w} = \mathbf{v}$ to create a new linear system:
\begin{equation*}
    \begin{bmatrix}
        \X^T - \frac{\one_n \one_n^T \mathbf{D}\X^T}{\Tr(\mathbf{D})} & -\mathbf{D}^{-1} \\ 
        -\mathbf{I}_d & \X
    \end{bmatrix}
    \begin{bmatrix}
        \mathbf{v} \\ \mathbf{u_x}
    \end{bmatrix}
    =
    \begin{bmatrix}
        \mathbf{b_w} + \mathbf{D}^{-1}\mathbf{b_z} - \frac{b_y + \one_n^T\left(\mathbf{D}\mathbf{b_w} + \mathbf{b_z}  \right)}{\Tr(D)} \one_n \\
    \mathbf{0}
    \end{bmatrix}.
\end{equation*}
Using the first block equation we can solve for $\mathbf{u_x}$: 
\begin{equation*}
    \mathbf{u_w} = \mathbf{D} \left[ \left(\X^T - \frac{\one_n \one_n^T D\X^T}{\Tr(\mathbf{D})}  \right)\mathbf{v} -  \left( \mathbf{b_w} + \mathbf{D}^{-1}\mathbf{b_z} - \frac{b_y + \one_n^T(\mathbf{D}\mathbf{b_w} + \mathbf{b_z}  )}{\Tr(\mathbf{D})} \one_n\right) \right].
\end{equation*}
Now we substitute this into the second equation to get a linear system for $\mathbf{v}$: 
\begin{equation*}
    \left(\X\mathbf{D}  \left(\X^T - \frac{\one_n \one_n^T \mathbf{D}\X^T}{\Tr(\mathbf{D})}\right) - \mathbf{I}_d  \right)\mathbf{v} = \X\mathbf{D}\left( \mathbf{b_w} + \mathbf{D}^{-1}\mathbf{b_z} - \frac{b_y + \one_n^T(\mathbf{D}\mathbf{b_w} + \mathbf{b_z}  )}{\Tr(\mathbf{D})} \one_n\right).
\end{equation*}
The above is now a linear system of $d$ variables which we can solve. After that, we can use the above derivation in reverse to compute $\mathbf{u_x}, u_y, \mathbf{u_z}$.

\end{document}